\documentclass[fleqn,10pt]{wlpeerj}

\let\temp\rmdefault
\usepackage{mathpazo}
\let\rmdefault\temp

\usepackage{url}
\usepackage{multirow}
\usepackage{lipsum} 
\usepackage{amsmath,amssymb,amsfonts}
\usepackage{amsthm}
\usepackage[hidelinks]{hyperref}
\usepackage{algorithm}
\usepackage{array}
\usepackage{algpseudocode}
\usepackage{graphicx}
\usepackage{textcomp}
\usepackage{xcolor}
\usepackage{physics}
\usepackage{multirow}
\usepackage{subcaption}
\usepackage{mathtools}
\usepackage{tikz}
\usetikzlibrary{calc}
\usepackage{zref-savepos}
\usepackage{pict2e}
\usepackage{babel}
\usepackage{pdfcomment}

\newtheorem{theorem}{Theorem}
\newtheorem{definition}{Definition}
\usepackage{pifont}
\usepackage[outline]{contour}
\newcommand{\cmark}{\contour{black}{\textcolor{green}{\ding{51}}}}
\newcommand{\xmark}{\contour{black}{\textcolor{red}{\ding{55}}}}

\newcounter{NoTableEntry}
\renewcommand*{\theNoTableEntry}{NTE-\the\value{NoTableEntry}}
\newcommand*{\notableentry}{%
  \multicolumn{1}{@{}c@{}|}{%
    \stepcounter{NoTableEntry}%
    \vadjust pre{\zsavepos{\theNoTableEntry t}}%
    \vadjust{\zsavepos{\theNoTableEntry b}}%
    \zsavepos{\theNoTableEntry l}%
    \hspace{0pt plus 1filll}%
    \zsavepos{\theNoTableEntry r}%
    \tikz[overlay]{%
      \draw[black]
        let
          \n{llx}={\zposx{\theNoTableEntry l}sp-\zposx{\theNoTableEntry r}sp},
          \n{urx}={0},
          \n{lly}={\zposy{\theNoTableEntry b}sp-\zposy{\theNoTableEntry r}sp},
          \n{ury}={\zposy{\theNoTableEntry t}sp-\zposy{\theNoTableEntry r}sp}
        in
        (\n{llx}, \n{ury}) -- (\n{urx}, \n{lly})
      ;
    }%
  }%
}

\newcolumntype{M}[1]{>{\centering\arraybackslash}m{#1}}

\title{Parameter efficient hybrid spiking-quantum convolutional neural network with surrogate gradient and quantum data-reupload}

\author[1, 3]{Luu Trong Nhan}
\author[2]{Luu Trung Duong}
\author[3]{Pham Ngoc Nam}
\author[4]{Truong Cong Thang}
\affil[1]{College of Information and Communication Technology, Can Tho University, Vietnam}
\affil[2]{Center of Digital Transformation and Communication, Can Tho University, Vietnam}
\affil[3]{College of Engineering and Computer Science, VinUniversity, Hanoi, Vietnam}
\affil[4]{Department of Computer Science and Engineering, The University of Aizu, Japan}
\corrauthor{Luu Trong Nhan}{luutn@ctu.edu.vn}

\begin{abstract}
The rapid advancement of artificial intelligence (AI) and deep learning (DL) has catalyzed the emergence of several optimization-driven subfields, notably neuromorphic computing and quantum machine learning. Leveraging the differentiable nature of hybrid models, researchers have explored their potential to address complex problems through unified optimization strategies. One such development is the Spiking Quantum Neural Network (SQNN), which combines principles from spiking neural networks (SNNs) and quantum computing. However, existing SQNN implementations often depend on pretrained SNNs due to the non-differentiable nature of spiking activity and the limited scalability of current SNN encoders. In this work, we propose a novel architecture—Spiking-Quantum Data Re-upload Convolutional Neural Network (SQDR-CNN)—that enables joint training of convolutional SNNs and quantum circuits within a single backpropagation framework. Unlike its predecessor, SQDR-CNN allow convergence to reasonable performance without the reliance of pretrained spiking encoder and subsetting datasets. We also clarified some theoretical foundations, testing new design using quantum data-reupload with different training algorithm-initialization and evaluate the performance of the proposed model under noisy simulated quantum environments. As a result, we were able to achieve 86\% of the mean top-performing accuracy of the SOTA SNN baselines, yet uses only 0.5\% of the smallest spiking model’s parameters. Through this integration of neuromorphic and quantum paradigms, we aim to open new research directions and foster technological progress in multi-modal, learnable systems.
\end{abstract}

\begin{document}
\flushbottom
\maketitle
\thispagestyle{empty}

\section{Introduction}
Recent advances in artificial intelligence (AI) and machine learning (ML) have transformed diverse areas of computer science. Breakthroughs such as GPT-4~\citep{peng2023instruction}, built on deep transformer-based architectures, have set new benchmarks in natural language processing and multitask learning. In parallel, diffusion-based generative models like Stable Diffusion~\citep{rombach2022high} have revolutionized image synthesis, producing diverse and high-resolution outputs. Reinforcement learning has also seen remarkable progress: systems such as AlphaCode~\citep{li2022competition} demonstrate the ability to solve complex coding problems, showcasing AI’s growing capacity for structured reasoning and problem solving.

Alongside these well-established directions, hybrid quantum–classical models are emerging as a promising frontier. By combining the strengths of quantum processors with conventional ML techniques, these approaches offer new opportunities, particularly for optimization and sampling tasks~\citep{de2022survey}. Early studies have explored quantum-enhanced neural models that operate efficiently in high-dimensional feature spaces~\citep{liu2021hybrid, fan2023hybrid}. Quantum machine learning tools such as the Variational Quantum Eigensolver (VQE)~\citep{tilly2022variational}, Hardware Efficient Ansatz~\citep{kandala2017hardware}, and data re-uploading frameworks~\citep{perez2020data} have already been incorporated into learning pipelines, paving the way for tackling problems beyond classical reach.

Spiking neural networks (SNNs)~\citep{yamazaki2022spiking, nunes2022spiking} represent another line of innovation. Inspired by biological neurons, SNNs communicate via discrete spikes, which makes them inherently energy-efficient due to their sparse activity~\citep{dampfhoffer2022snns, eshraghian2023training}. A variety of training strategies have been developed, including backpropagation through time (BPTT)~\citep{yang2019dashnet}, ANN-to-SNN conversion~\citep{nunes2022spiking}, spike-timing-dependent plasticity (STDP)~\citep{markram2011history}, and surrogate gradient methods~\citep{neftci2019surrogate, fang2021deep, fang2023spikingjelly}, all of which have improved performance on spatio-temporal tasks. More recently, hybrid architectures that integrate artificial neural networks (ANNs) with SNNs~\citep{muramatsu2023combining, chen2023hybrid, yang2019dashnet, seekings2024towards} have demonstrated the benefits of combining dense ANN feature extractors with sparsely activated SNN layers, achieving strong results in applications such as image classification~\citep{muramatsu2023combining, chen2023hybrid, seekings2024towards} and object detection~\citep{yang2019dashnet}.

An emerging area of research brings SNNs together with parameterized quantum circuits (PQCs), leveraging both the temporal encoding efficiency of spiking models and the high-dimensional computational power of quantum systems. For instance, \citet{xu2024parallel} proposed a parallel hybrid model where spiking and quantum modules operate simultaneously, improving robustness and accuracy. Other efforts include the development of quantum leaky integrate-and-fire (LIF) neurons~\citep{brand2024quantum}, enabling scalable quantum SNNs, and hybrid classifiers that use SNNs for temporal encoding with quantum circuits for projection into higher-dimensional spaces~\citep{ajayan2021edge}.

One notable recent contribution in this domain is the shallow hybrid quantum–spiking architecture proposed by~\citet{konar2023shallow}, which showed strong performance on standard vision benchmarks, particularly under noisy conditions, outperforming both CNNs and SNNs. However, existing studies fall short of providing a comprehensive theoretical framework and systematic benchmarking of proposed model under diverse training conditions. Combined with recent advances in PQC design, this gap motivates us to further develop new models building upon this foundational knowledge. Building upon this prior work, our study introduces a more advanced architecture—\textbf{Spiking-Quantum Data Re-upload Convolutional Neural Network (SQDR-CNN)}—that incorporates enhanced optimization strategies and modern architectural innovations. This model integrates spiking convolutional layers trained via surrogate gradients~\citep{fang2021deep} with a quantum data re-uploading classifier~\citep{perez2020data}, which acts as a universal function approximator within the quantum domain. Our main contributions are summarized as follows:
\begin{itemize}
    \item We introduce SQDR-CNN, a theoretically motivated architecture that fuses spiking convolutional networks with quantum classifiers based on data re-uploading techniques (which was unexplored by prior researches under hybrid spiking-quantum model design), yielding strong performance across standard benchmarks. As a result, we were able to achieve 86\% of the mean top-performing accuracy of the SOTA SNN baselines, yet uses only 0.5\% of the smallest spiking model’s parameters.
    \item We benchmark our model against established SNN architectures, demonstrating that SQDR-CNN achieves comparable accuracy with a considerably smaller parameter set.
    \item We establish theoretical basis for some practice that are utilized in prior work without clear explanation.
    \item We evaluate the robustness and generalization capacity of SQDR-CNN under varying training regimes, including different optimizers, noisy circuit components and initialization schemes.
\end{itemize}

A more comprehensive comparison of our contributions with respect to prior works can be summarized in Table \ref{tab:contr}.
\begin{table*}[t!]
    \centering
    \caption{Contribution comparison table of prior works with respect to ours.}
    \resizebox{\textwidth}{!}{
    \begin{tabular}{|M{1.8cm}|*{5}{M{2.5cm}|}} 
    \hline
    \multirow{3}{*}{\shortstack[c]{Related\\research}} & \multicolumn{5}{c|}{Contributions}\\
    \cline{2-6}
     & Training with e2e spike-based backprop & Tackled full dataset & NISQ noise models benchmarking & Initialization benchmarking & Optimization algorithm benchmarking\\
     \hline
     \citet{xu2024parallel} & \cmark & \cmark & \xmark & \xmark & \xmark \\
    \hline
     \citet{ajayan2021edge} & \xmark & \cmark & \cmark & \xmark & \xmark \\
    \hline
    \citet{sun2021quantum} & \cmark & \cmark & \xmark & \xmark & \xmark \\
    \hline
    \citet{kristensen2021artificial} & \xmark & \xmark & \xmark & \xmark & \xmark \\
    \hline
    \citet{konar2023shallow} & \cmark & \xmark & \xmark & \xmark & \xmark \\
    \hline
    \bf Our work & \cmark & \cmark & \cmark & \cmark & \cmark \\
    \hline
    \end{tabular}}
    \label{tab:contr}
\end{table*}
\section{Related works}
Quantum information science (QIS) has emerged as a promising field with the potential to significantly accelerate computations across various disciplines, including machine learning and neural network models~\citep{konar2023shallow, luu2024universal}. Yet, realizing powerful quantum neural networks is hindered by the limitations of present-day quantum hardware, particularly due to decoherence and noise in quantum gates~\citep{schuld2015introduction}. Consequently, current research has increasingly turned towards hybrid quantum–classical approaches, which are better suited to the constraints of Noisy Intermediate-Scale Quantum (NISQ) devices~\citep{preskill2018quantum}, allowing for efficient processing of high-dimensional inputs even with a limited number of qubits.

Variational quantum circuits (VQCs) form the backbone of many such hybrid learning algorithms and have shown practical effectiveness across a range of tasks despite the hardware constraints. Nevertheless, the structural complexity of conventional quantum neural architectures—such as quantum recurrent networks~\citep{bausch2020recurrent}, quantum Hopfield networks~\citep{rebentrost2018quantum}, and quantum convolutional models~\citep{cong2019quantum, henderson2020quanvolutional}—continues to present obstacles for scalable quantum information processing.

Recent efforts have explored the integration of quantum computing principles with SNNs, resulting in several early-stage frameworks. For instance, \citet{kristensen2021artificial} proposed a quantum neuron model based on the time evolution of tunable Hamiltonians, enabling signal transformation via quantum amplitudes and measurement-induced changes. Similarly, \citet{sun2021quantum} introduced the Quantum Superposition Spiking Neural Network (QS-SNN), which encodes spatiotemporal spikes within quantum states to improve robustness under noise. Another approach by \citet{ajayan2021edge} combined SNN dynamics with quantum computation in a hybrid image classification framework.

Despite these conceptual advancements, most existing models~\citep{kristensen2021artificial, sun2021quantum, ajayan2021edge} are theoretical in nature and have not yet been deployed on actual quantum processors or full-scale simulators. More applied work includes the quantum circuit-based SNN model by \citet{chen2022accelerating}, tested on binary MNIST classification \citep{deng2012mnist}, and the hybrid SQNN model by \citet{konar2023shallow}, which incorporated a pretrained SNN encoder for image data. However, both approaches are still limited by scalability issues and the current hardware restrictions on circuit depth and qubit availability.

\section{Architecture}
\subsection{SNN Encoder}
The incorporation of convolutional layers into hybrid quantum models has become a common approach to reduce the input dimensionality of vision datasets \citep{pan2023hybrid, liu2021hybrid, mishra2023qsurfnet}. This integration allows for efficient feature extraction, enhancing the performance of hybrid quantum neural networks in image processing tasks. Convolutional layers capture local spatial hierarchies in data, reducing dimensionality while retaining critical features. As a result, hybrid quantum-classical architectures leveraging convolutional layers can benefit from improved representation learning, enabling better generalization and robustness in various applications.

To encapsulate image information using convolution operations, we employ a standard convolutional encoder design inspired by well known image processing architecture \citep{krizhevsky2012imagenet, szegedy2015going, simonyan2014very}, with layers of convolution followed by batch normalization and max pooling (for more information please refer to Figure \ref{fig:model}). This structured approach ensures that the extracted features are well-normalized and efficiently reduced in size. The final feature maps are then converted into a compact vector representation using adaptive pooling \citep{gholamalinezhad2020pooling}, which is subsequently forwarded to the PQC model. 

\begin{figure*}[t]
\centerline{\includegraphics[width=\linewidth]{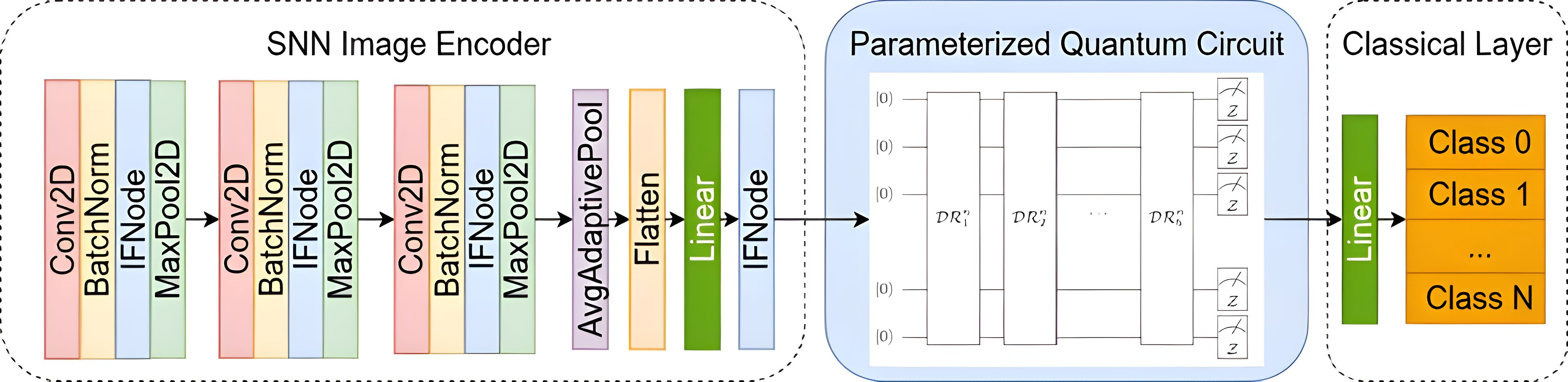}}
\caption{Architecture of our SQDR-CNN. Inputs are first processed by the SNN-based encoder, generating a feature vector that is passed to our PQC for computation. The measurement results from the PQC are then fed into a final classical MLP layer, which produces the prediction output.}
\label{fig:model}
\end{figure*}

Furthermore, all weight kernels used in the encoder are trained with surrogate gradient \citep{fang2021deep} to all Integrate-and-Fire (IF) neurons, utilizing the arctan surrogate function, which had been deeply studied in prior works for it efficiency and performance in the context of surrogate training \citep{fang2021deep, eshraghian2023training, fang2021incorporating, luu2025hybrid, luu2025accuracy}. Its derivative, $\partial \Theta/\partial x$, for the Heaviside function $\Theta(x, V_{th})$ with potential voltage $V_{th}$ and input $x$ during backpropagation is defined as:
\begin{equation}
    \begin{aligned}
        \Theta(x, V_{th}) &=
        \begin{cases}
        1 & \text{if } x - V_{th} \geq 0 \\
        0 & \text{otherwise}
        \end{cases}, \quad \frac{\partial \Theta}{\partial x} = \frac{\alpha}{2\left( 1 + \left( \frac{\pi}{2} \alpha x \right)^2 \right)} \quad \text{where $\alpha=2$}
        \end{aligned}
    \label{heviside}
\end{equation}
The use of spiking neurons introduces biologically inspired dynamics, which can enhance feature extraction and learning efficiency \citep{fang2021deep, eshraghian2023training, fang2021incorporating, nunes2022spiking, clanuwat2018deep, cheng2020lisnn}. The surrogate gradient method allows for effective backpropagation in spiking neural networks, facilitating stable and efficient training within the hybrid quantum-classical framework.

\subsection{Quantum data-reupload PQC} \label{pqc_design}
A critical challenge in designing an efficient architecture lies in identifying an optimal PQC structure. Poorly chosen PQC designs are susceptible to the quantum barren plateau problem \citep{mcclean2018barren}, make model training inefficient and hinder overall performance in hybrid quantum design. Addressing this issue, our work explores various PQC designs that can function as effective feature classifiers while mitigating the barren plateau phenomenon.

Among many commonly studied PQC architectures, the Hardware Efficient Ansatz (HEA) \citep{kandala2017hardware, park2024hardware, leone2024practical} stand out for their widely usage as one of the founding optimizable quantum model. HEA prioritizes hardware adaptability by using shallow circuit depths and simple rotation unitaries $\hat{u}$, the formula of a HEA $\hat{U}_{\text{HEA}}$ can be written as:
\begin{equation}
    \begin{aligned}
        & \hat{U}_{\text{HEA}}(\boldsymbol{\theta}) \\
        & = 
        \prod_{i=1}^{N_q} \hat{u}_{R_z}(\theta_{i, N_L})
        \times
        \hat{u}_{\text{Ent}}
        \times
        \prod_{i=1}^{N_q} \hat{u}_{R_x}(\theta_{i, N_L-1})
        \times
        \hat{u}_{\text{Ent}} \cdots
        \prod_{i=1}^{N_q} \hat{u}_{R_z}(\theta_{i, 1})
        \times
        \hat{u}_{\text{Ent}} 
        \times
        \prod_{i=1}^{N_q} \hat{u}_{R_z}(\theta_{i, \text{cap}})
    \end{aligned}
\end{equation}
where \( N_q \) is the number of qubits, \( N_L \) is the number of layers, and index \( l \) labels the layers (indexing is in reverse order). The final layer (\( l = 1 \)) is terminated with another set of single qubit rotations (the cap or "capping" block).

We also have to consider the original PQC design of SQNN from \citet{konar2023shallow} where the initial unitary group \( U_{\text{enc}} \) applied angular encoding to the initial states \( |0\rangle^{\otimes N} \) independently:
\[
U_{\text{enc}}(\boldsymbol{\omega}) = \bigotimes_{i=0}^{N} \left( R_Z(\omega_i) \cdot H \right)
\]
The superposition states are then forward to trainable unitaries layer $U_{\text{train}}$:
\[
U_{\text{train}}(\boldsymbol{\omega}) = \prod_{(i,j) \in E} \text{CNOT}_{i,j} \cdot \left( \bigotimes_{k=0}^{N} \mathbb{M}_k(\omega_k) \right)
\]
where unitary \( \mathbb{M}(\omega) = R_Z(\omega) R_X(\omega) \) and \( E \) is the set of control-target qubit indices corresponding to the CNOT connections in the circuit. The final circuit is given by:
\[
U(\boldsymbol{\omega}) = U_{\text{train}}(\boldsymbol{\omega}) \cdot U_{\text{enc}}(\boldsymbol{\omega})
\]

Beyond existing designs, we propose exploring a novel PQC framework: quantum reuploading \citep{perez2020data}. This approach has demonstrated significant potential in prior studies, particularly in quantum data reupload applications. Notably, researches had shown that quantum reuploading enabled efficient multiclass classification using a single qubit \citep{wach2023data, easom2021depth}. By repeatedly encoding classical data into quantum states, the method overcomes the data limitation of small-scale quantum systems, providing an effective and resource-efficient classifier design. 
\begin{equation}
    U(\vec{\phi}, \vec{x}) \equiv U(\phi_1) U(\vec{x}) \cdots U(\phi_N) U(\vec{x})
    \label{qdr}
\end{equation}

As noted in previous research \citep{perez2020data}, single-qubit data reupload classifier introduced can approximate any classification function up to arbitrary precision based on Universal Approximation Theorem (UAT) \citep{perez2020data, hornik1991approximation}:
\begin{theorem}[Hornik et al.]
Let \( I_m = [0, 1]^m \) denote the \( m \)-dimensional unit cube, and let \( \mathcal{C}(I_m) \) represent the space of continuous real-valued functions on \( I_m \). Consider a function \( \varphi: \mathbb{R} \to \mathbb{R} \) that is nonconstant, bounded, and continuous. For any continuous target function \( f : I_m \to \mathbb{R} \) and for any desired approximation accuracy \( \varepsilon > 0 \), there exists an integer \( N \in \mathbb{N} \) and a function \( h : I_m \to \mathbb{R} \) of the form:
\[
h(\vec{x}) = \sum_{i=1}^N \alpha_i \, \varphi(\vec{w}_i \cdot \vec{x} + b_i),
\]
where \( \alpha_i, b_i \in \mathbb{R} \) and \( \vec{w}_i \in \mathbb{R}^m \) for all \( i = 1, \dots, N \), such that the function \( h \) approximates \( f \) to within \( \varepsilon \) uniformly over \( I_m \):
\[
|h(\vec{x}) - f(\vec{x})| < \varepsilon, \quad \text{for all } \vec{x} \in I_m.
\]
\end{theorem}
This makes quantum reuploading a compelling candidate for integration into our hybrid SQDR-CNN model, with potential implications for future PQC advancements.

In our proposed model, we construct the $n$-qubit quantum data reupload block, denoted as $\mathcal{DR}$, using $Rot$ gates \citep{nielsen2010quantum}. Traditionally, these gates perform arbitrary single-qubit rotations, defined as:
\begin{equation}
\begin{aligned}
    & U(\phi_N) = Rot(\alpha, \beta, \gamma, \sigma) = e^{i\alpha}RZ(\beta) RY(\gamma) RZ(\sigma) \\
    &=
        \begin{bmatrix}
        e^{i(\alpha-\beta/2-\sigma/2)} \cos(\gamma/2) & -e^{i(\alpha-\beta/2+\sigma/2)} \sin(\gamma/2) \\
        e^{i(\alpha+\beta/2-\sigma/2)} \sin(\gamma/2) & e^{i(\alpha+\beta/2+\sigma/2)} \cos(\gamma/2)
        \end{bmatrix}
\end{aligned}
\end{equation}
to encode entries of input feature vectors and construct optimizable parameterized gates. In our experiments, we used a quantized version of the $Rot$ gate to reduce parameter count and computational load \citep{bergholm2018pennylane}, defined as:
\begin{equation}
\begin{aligned}
    & U(\phi_N) = Rot(\omega, \theta, \phi) = RZ(\omega) RY(\theta) RZ(\phi) \\
    &=
        \begin{bmatrix}
        e^{-i(\phi+\omega)/2} \cos(\theta/2) & -e^{i(\phi-\omega)/2} \sin(\theta/2) \\
        e^{-i(\phi-\omega)/2} \sin(\theta/2) & e^{i(\phi+\omega)/2} \cos(\theta/2)
        \end{bmatrix}.
\end{aligned}
\end{equation}
These rotations are followed by a series of circularly entangled controlled-Z ($CZ$) gates. For a more detailed illustration, refer to Figure \ref{fig:dr_block}.

\begin{figure}[htbp]
\centerline{\includegraphics[width=0.8\linewidth]{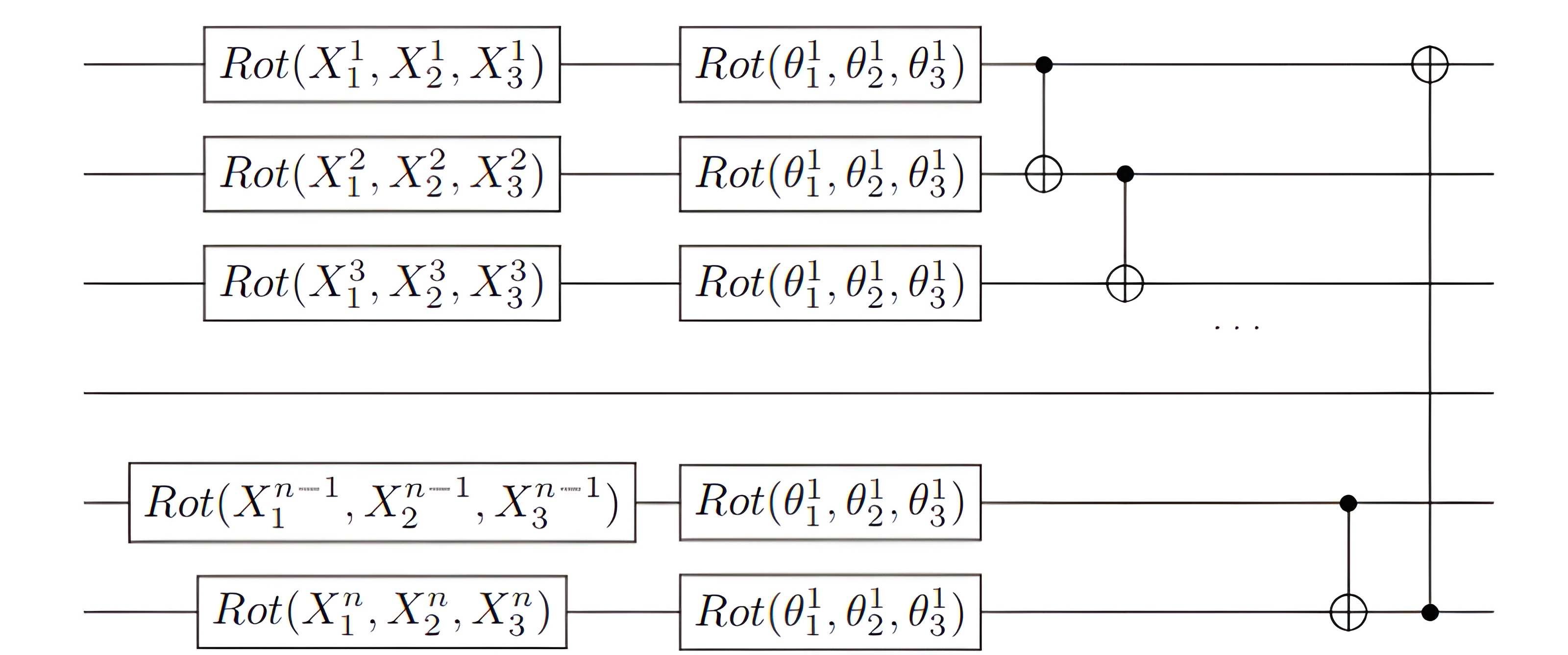}}
\caption{Architecture of an $n$ qubit data reupload $\mathcal{DR}$ block used in our experiment. First layer of $Rot$ gates are used to encode data from feature vector $X$ received from image encoder while gates in the second layer serve as trainable parameters, followed by circularly arranged $CZ$ gates.}
\label{fig:dr_block}
\end{figure}

To make our PQC scalable not only in terms of the number of qubits but also in terms of gate depth, we stack multiple $\mathcal{DR}$ blocks sequentially, thereby enhancing the model’s expressive power and processing capability (see Figure \ref{fig:circuit} for details). After the quantum transformation, we extract classical information from the quantum state by performing measurements in the computational basis using the Pauli-Z observable. For an $n$-qubit PQC, the Pauli-Z measurement on qubit \( j \) corresponds to measuring the observable:
\begin{equation}
    \begin{aligned}
        Z &= \begin{bmatrix} 1 & 0 \\ 0 & -1 \end{bmatrix}, \quad Z_j = I \otimes \cdots \otimes I \otimes Z \otimes I \otimes \cdots \otimes I = I^{\otimes (j-1)} \otimes Z \otimes I^{\otimes (n-j)}.
    \end{aligned}
\end{equation}

\begin{figure}[t]
\centerline{\includegraphics[width=0.8\linewidth]{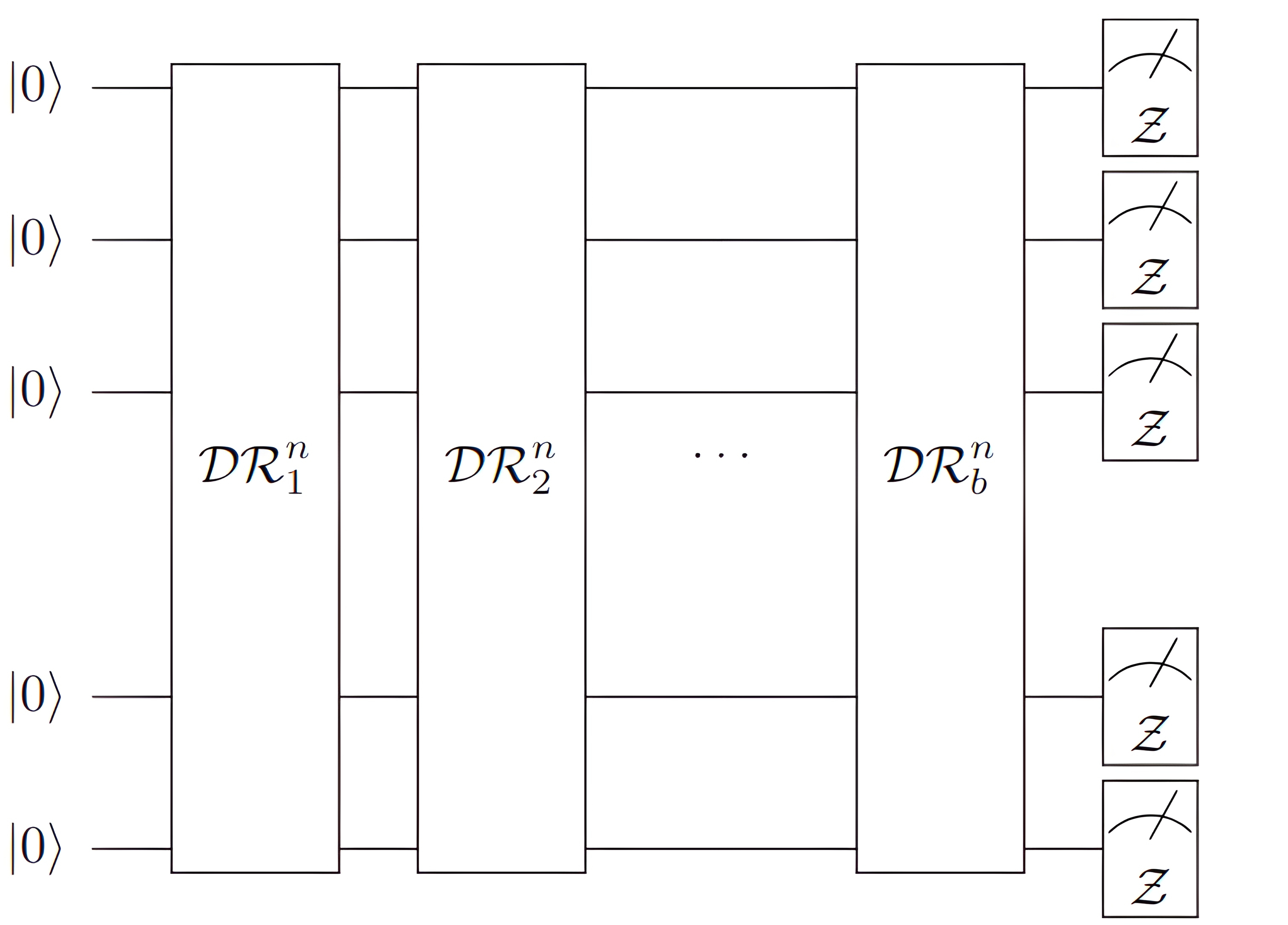}}
\caption{Architecture of the PQC used in our model, consist of $b$ blocks of $n$ qubit data reupload $\mathcal{DR}$ layer. Qubits are then measured using Pauli-Z basis to obtain results for classical computing.}
\label{fig:circuit}
\end{figure}

Aside from theoretical justification for PQC selection, we also practically justify it by performing evaluation against established PQC architectures on MNIST \citep{deng2012mnist} dataset. As summarized in Table \ref{tab:pqc_compare}, our empirical results demonstrate that the data re-upload circuit consistently outperformed the other PQCs across multiple tasks, followed by SQNN's PQC implementation and HEA with lowest performance. Average performance gap range from 1.7\% (between SQNN's PQC and HEA) to 3.61\% (between data re-upload and HEA).

\begin{table}[t]
    \centering
    \caption{Comparison of SQDR-CNN validation accuracy on MNIST \citep{deng2012mnist} using various PQC architecture, best performances are denoted in bold. Quantum reupload \citep{perez2020data} as proposed in our paper performed the best.}
    \resizebox{\textwidth}{!}{
    \begin{tabular}{|M{7cm}|*{4}{M{1.6cm}|}} 
    \hline
     \multirow{2}{*}{Variants} & \multicolumn{4}{c|}{Validation accuracy (\%)} \\
    \cline{2-5}
     & MNIST \citep{deng2012mnist}& Fashion-MNIST \citep{xiao2017fashion} & KMNIST \citep{clanuwat2018deep} & Average\\
    \hline
    HEA(9-qubit) \citep{kandala2017hardware} & 80.30 & 72.85 & 76.26 & 76.47\\
    \hline
    SQNN(9-qubit) \citep{konar2023shallow} & 82.45 & 74.50 & 77.55 & 78.17\\
    \hline
    Quantum data-reupload[2b-9q] \citep{perez2020data} & \textbf{84.03} & \textbf{76.18} & \textbf{80.04} & \textbf{80.08}\\
    \hline
    \end{tabular}}
    \label{tab:pqc_compare}
\end{table}

\subsection{Post processing with classical multi-layer perceptron (MLP)}
Following the PQC processing stage, the resulting output feature vector is passed to a classical MLP for further transformation and prediction. Although this practice while has been extensively explored in prior research \cite{konar2023shallow, alam2022deepqmlp}, theoretical basis for this practice was not clearly explained. We hypothesize that incorporating a trainable classical layer after the PQC offers two key advantages: 
\begin{itemize}
    \item It introduces additional learnable parameters that can enhance model expressivity and overall performance.
    \item The presence of a post-processing MLP layer enables effective gradient rescaling during backpropagation. This can amplify variance in quantum gradients and help reduce the the barren plateau phenomenon \cite{cerezo2021cost} often observed in PQC training.
\end{itemize} 
To further solidify our argument, we will proceed to prove the mentioned views analytically. First, we define some terms that would be used in our analysis. Quantum barren plateau can be formally quantified as \citep{larocca2025barren}:
\begin{definition}[Probabilistic concentration of weight gradient]
The resulted measurement $f_{\theta}(\rho, O)$ of input state $\rho$ and observable $O$ of a $n$-qubit PQC is said to exhibits a probabilistic barren plateau over an exponent base $b$ if:
    \begin{equation}
        \mathrm{Var}_{\theta}[f_{\theta}(\rho,O)] \;\; \text{or} \;\; \mathrm{Var}_{\theta}[\partial_{\mu}f_{\theta}(\rho,O)] \in \mathcal{O}\left(\frac{1}{b^n}\right),
        \label{eq:prob_bp}
    \end{equation}
for all parameter $\theta$ or a subset of parameters $\theta_\mu \in \theta$ and the value of $b$ depend on initial state $\rho$ and observable $O$.
\label{barren_plateau_def}
\end{definition}
We can then proceed to prove our statement as follow:
\begin{theorem}[Concentration reduction over shared variance]
Given a PQC with $\phi_i$ is the $i$-th $\in n$ parameter of the PQC, and $\psi_{out}(\phi, x)$ represents the applied unitary transformations in the quantum circuit, and $\phi_i$ gradient is derived through parameter shift rule \citep{wierichs2022general} as:
\begin{equation}
    \begin{aligned}
        &\frac{\partial \mathcal{L}}{\partial \phi_i} = \frac{\partial \mathcal{L}}{\partial y} \cdot \frac{\partial y}{\partial f} \cdot \frac{\partial f}{\partial \phi_i}, \\
        &\frac{\partial y(f(\phi, x))}{\partial f(\phi, x)} = \frac{\partial (\boldsymbol{W}^\top f(\phi, x) + \boldsymbol{b})}{\partial f(\phi, x)} = \boldsymbol{W}^\top, \\
        & \frac{\partial f(\phi, x)}{\partial \phi_i}  =  \frac{\partial \langle \psi_{out}(\phi_i, x) | \hat{O} | \psi_{out}(\phi_i, x) \rangle}{\partial \phi_i} = \frac{1}{2} \left[ f(\phi_{+}^{(i)}, x) - f(\phi_{-}^{(i)}, x) \right],\\
    \end{aligned}
\end{equation}
where $\mathcal{L}$ denotes the loss function, $y$ is the MLP output, $f(\phi, x)$ is the positive operator-valued measure (POVM) with POVM element $\hat{O}$. Assuming $\operatorname{Var}\left(\frac{\partial y(f(\phi, x))}{\partial f(\phi, x)}\right) + \mathbb{E}\left[\frac{\partial y(f(\phi, x))}{\partial f(\phi, x)}\right]^2 \geq 1$ hold and $\frac{\partial y(f(\phi, x))}{\partial f(\phi, x)}$ and $\frac{\partial f(\phi, x)}{\partial \phi_i}$ are independent variable, then we would have:
\begin{equation}
    \begin{aligned}
        & \mathrm{Var}\left[\frac{\partial y(f(\phi, x))}{\partial f(\phi, x)}\cdot\frac{\partial f(\phi, x)}{\partial \phi_i}\right] \geq \mathrm{Var}\left[\frac{\partial f(\phi, x)}{\partial \phi_i}\right]\\
    \end{aligned}
    \label{grad_w}
\end{equation}
\end{theorem}
\begin{proof}
Deriving $\mathrm{Var}\left[\frac{\partial y(f(\phi, x))}{\partial f(\phi, x)}\cdot\frac{\partial f(\phi, x)}{\partial \phi_i}\right]$, we would have:
\begin{equation}
    \begin{aligned}
        & \mathrm{Var}\left[\frac{\partial y(f(\phi, x))}{\partial f(\phi, x)}\cdot\frac{\partial f(\phi, x)}{\partial \phi_i}\right] \\
        = \quad & \mathrm{Cov}\left(\left(\frac{\partial y(f(\phi, x))}{\partial f(\phi, x)}\right)^2, \left(\frac{\partial f(\phi, x)}{\partial \phi_i}\right)^2\right)\\
        & + \left(\operatorname{Var}\left(\frac{\partial y(f(\phi, x))}{\partial f(\phi, x)}\right) + \mathbb{E}\left[\frac{\partial y(f(\phi, x))}{\partial f(\phi, x)}\right]^2\right) \cdot \left(\operatorname{Var}\left(\frac{\partial f(\phi, x)}{\partial \phi_i}\right) + \mathbb{E}\left[\frac{\partial f(\phi, x)}{\partial \phi_i}\right]^2\right)\\
        & -  \left[\mathrm{Cov}\left(\left(\frac{\partial y(f(\phi, x))}{\partial f(\phi, x)}\right)^2, \left(\frac{\partial f(\phi, x)}{\partial \phi_i}\right)^2\right) + \mathbb{E}\left[\frac{\partial y(f(\phi, x))}{\partial f(\phi, x)}\right]\mathbb{E}\left[\frac{\partial f(\phi, x)}{\partial \phi_i}\right]\right]^2\\
    \end{aligned}
\end{equation}
Since $\frac{\partial y(f(\phi, x))}{\partial f(\phi, x)}$ and $\frac{\partial f(\phi, x)}{\partial \phi_i}$ is independent, this simplified to:
\begin{equation}
    \begin{aligned}
            & \mathrm{Var}\left[\frac{\partial y(f(\phi, x))}{\partial f(\phi, x)}\cdot\frac{\partial f(\phi, x)}{\partial \phi_i}\right]\\
            = \quad & \left( \operatorname{Var}\left(\frac{\partial y(f(\phi, x))}{\partial f(\phi, x)}\right) + \mathbb{E}\left[\frac{\partial y(f(\phi, x))}{\partial f(\phi, x)}\right]^2 \right)\cdot \left( \operatorname{Var}\left(\frac{\partial f(\phi, x)}{\partial \phi_i}\right) + \mathbb{E}\left[\frac{\partial f(\phi, x)}{\partial \phi_i}\right]^2 \right)\\
            & - \left[\mathbb{E}\left[\frac{\partial y(f(\phi, x))}{\partial f(\phi, x)}\right]\mathbb{E}\left[\frac{\partial f(\phi, x)}{\partial \phi_i}\right]\right]^2\\
            = \quad & \operatorname{Var}\left(\frac{\partial y(f(\phi, x))}{\partial f(\phi, x)}\right)\cdot \operatorname{Var}\left(\frac{\partial f(\phi, x)}{\partial \phi_i}\right) + \operatorname{Var}\left(\frac{\partial y(f(\phi, x))}{\partial f(\phi, x)}\right)\cdot \mathbb{E}\left[\frac{\partial f(\phi, x)}{\partial \phi_i}\right]^2\\
            & + \operatorname{Var}\left(\frac{\partial f(\phi, x)}{\partial \phi_i}\right)\cdot \mathbb{E}\left[\frac{\partial y(f(\phi, x))}{\partial f(\phi, x)}\right]^2\\
        \end{aligned}
\end{equation}
Then, decomposing the inequality we want to prove in theorem as:
\begin{equation}
    \begin{aligned}
        & \mathrm{Var}\left[\frac{\partial y(f(\phi, x))}{\partial f(\phi, x)}\cdot\frac{\partial f(\phi, x)}{\partial \phi_i}\right] \geq \mathrm{Var}\left[\frac{\partial f(\phi, x)}{\partial \phi_i}\right]\\ \\
        \Leftrightarrow \quad & \left(\begin{aligned}
            & \operatorname{Var}\left(\frac{\partial y(f(\phi, x))}{\partial f(\phi, x)}\right)\cdot \operatorname{Var}\left(\frac{\partial f(\phi, x)}{\partial \phi_i}\right)\\
            & + \operatorname{Var}\left(\frac{\partial y(f(\phi, x))}{\partial f(\phi, x)}\right)\cdot \mathbb{E}\left[\frac{\partial f(\phi, x)}{\partial \phi_i}\right]^2\\
            & + \operatorname{Var}\left(\frac{\partial f(\phi, x)}{\partial \phi_i}\right)\cdot \mathbb{E}\left[\frac{\partial y(f(\phi, x))}{\partial f(\phi, x)}\right]^2\\
        \end{aligned}\right) \geq \mathrm{Var}\left[\frac{\partial f(\phi, x)}{\partial \phi_i}\right] \\ \\
        \Leftrightarrow \quad & \mathrm{Var}\left[\frac{\partial f(\phi, x)}{\partial \phi_i}\right] \left(\begin{aligned}
            & \mathrm{Var}\left[\frac{\partial y(f(\phi, x))}{\partial f(\phi, x)}\right]\\
            & + \mathbb{E}\left[\frac{\partial y(f(\phi, x))}{\partial f(\phi, x)}\right]^2 - 1\\
        \end{aligned}\right) \geq -\operatorname{Var}\left[\frac{\partial y(f(\phi, x))}{\partial f(\phi, x)}\right]\cdot\mathbb{E}\left[\frac{\partial f(\phi, x)}{\partial \phi_i}\right]^2\\
    \end{aligned}
\end{equation}

Given the initial condition $\operatorname{Var}\left(\frac{\partial y(f(\phi, x))}{\partial f(\phi, x)}\right) + \mathbb{E}\left[\frac{\partial y(f(\phi, x))}{\partial f(\phi, x)}\right]^2 \geq 1$, the inequality always hold.\footnote{This initial condition can be considered somewhat reasonable, since complete saturation of gradient magnitude and variance are relatively rare in the context of ANN stochastic optimization \citep{gorbunov2020unified}.}
\end{proof}

As we can see from Equation \ref{grad_w}, optimized PQC gradient scale with post-processing MLP layer gradient $\partial y / \partial f$, which could contribute to rescaling PQC variance and lessen quantum barren plateau effect on PQC weights as long as gradient saturation does not happen. 

\begin{figure}[t]
\centerline{\includegraphics[width=\linewidth]{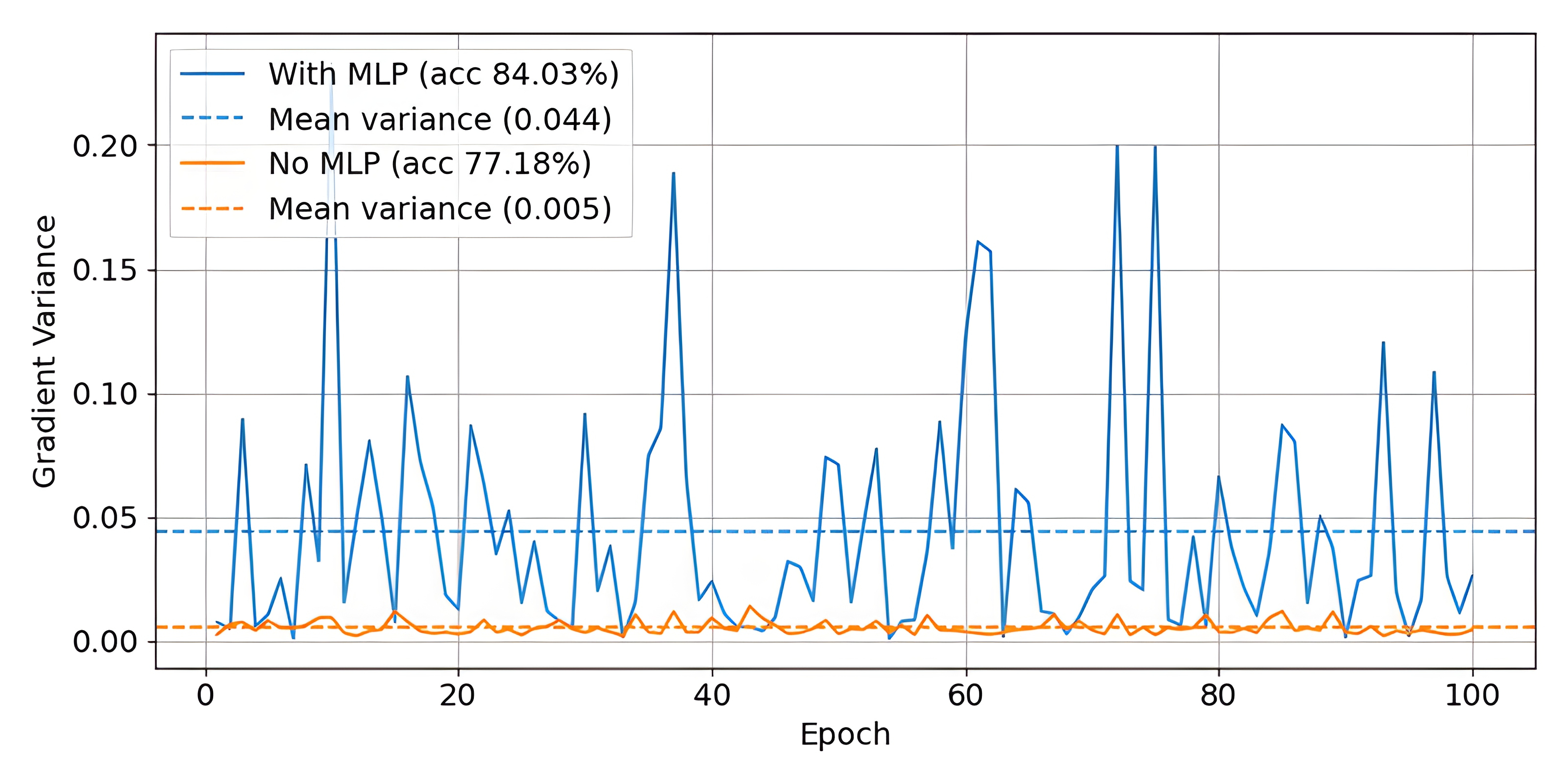}}
\caption{Comparison of PQC layer gradient variance when computed with and without MLP post-processing when training on MNIST \citep{deng2012mnist}. Dashed line denote mean value across all training epoches.}
\label{fig:norm_compare}
\end{figure}

To empirically validate our hypothesis, we measured the gradient variance of the PQC under two conditions: with and without the inclusion of a classical MLP as a post-processing module (see Figure~\ref{fig:norm_compare} for details). The results indicate that models incorporating the MLP not only achieve higher classification accuracy, but also exhibit consistently larger gradient variance within the PQC layers. This observation aligns with our hypothesis that the MLP facilitates gradient flow and enhances trainability of the quantum circuit.

\section{Experimental settings}
Majority of experiments were conducted using the PyTorch framework~ \citep{paszke2017automatic}, with SpikingJelly~ \citep{fang2023spikingjelly} employed for the SNN encoder. Quantum circuit simulations were performed using PennyLane~ \citep{bergholm2018pennylane} both for noiseless and noisy simulations.

For data preprocessing method, we performed image normalization within the range of [0, 1], employed horizontal random flip during training as augmentation method \citep{wang2017effectiveness} for better generalization and image resolution was kept as in original dataset. For spike-based encoding, we adopted the \textit{phase coding} scheme~ \citep{kim2018deep} with a time step size of \( T = 10 \), the corresponding spike weight \( \omega(t) \) is defined as:
\[
\omega(t) = 2^{-(1 + \operatorname{mod}(t - 1, K))}.
\]

Model optimization follows conventional deep learning practices~ \citep{luu2023blind, luu2024, luu2024improvement}, using the Adam optimizer~ \citep{kingma2014adam} with a learning rate of \( \text{lr} = 10^{-3} \), weight decay $\gamma= 10^{-3}$ and momentum parameters \( \beta = (0.9, 0.999) \), consistently applied across all datasets.

The loss function used was the Cross-Entropy Loss with Softmax activation and mean reduction over mini-batches~ \citep{paszke2017automatic}:
\begin{equation}
\begin{aligned}
    L(x, y) &= \frac{1}{N} \sum_{n=1}^{N} l_n(x, y), \\
    l_n(x, y) &= -\sum_{c=1}^{C} w_c \log\left( \frac{\exp(x_{n,c})}{\sum_{i=1}^{C} \exp(x_{n,i})} \right) y_{n,c}
\end{aligned}
\end{equation}
Here, \( x \) and \( y \) denote the model output and ground-truth labels respectively, \( w \) is the class-specific weighting factor, \( C \) is the total number of classes, and \( N \) is the mini-batch size.

Dataset splits follow the original protocols as described in their respective publications. For datasets without predefined splits, an 80/20 training/validation ratio is applied. All experiments were trained with fixed random seeds \citep{picard2021torch} for reproducibility for 100 epochs on an NVIDIA RTX 3060 GPU 12 GB.

To asses the effectiveness of our model in comparison with other model in image classification, the \textit{Accuracy} is one of the most commonly used metrics to evaluate the performance of a model. It measures the proportion of correctly predicted labels $y_i$ over the total number of predictions $\hat{y}_i$ within $N$ samples and can be computed as:
\[
\textit{Accuracy} = \frac{1}{N} \sum_{i=1}^{N} \mathbb{1}(\hat{y}_i = y_i)
\]
where $\mathbb{1}(\cdot)$ is the indicator function that returns 1 if the condition is true, and 0 otherwise.

\section{Results}

\subsection{Noiseless benchmarking with SOTA methods}

To prove that SQDR-CNN can utilize quantum supremacy to compete with traditional SNN solutions, we evaluate our architecture in noiseless scenario against a selection of prominent SOTA models, include ST-RSBP \citep{zhang2019spike}, SEW-ResNet18 \citep{fang2021deep}, LISNN \citep{cheng2020lisnn}, PLIF \citep{fang2021incorporating} and Spiking-ResNet \citep{hu2021spiking}. For each SQDR-CNN variant we adopt the notation “\textit{x}b–\textit{y}q”, where \textit{x} denotes the number of data‐reupload blocks and \textit{y} indicates the number of qubits employed in the PQC. We conduct our experiments with cross-dataset testing on three well‐known grayscale image classification benchmarks: MNIST \citep{deng2012mnist}, Fashion-MNIST \citep{xiao2017fashion}, and KMNIST \citep{clanuwat2018deep}. These datasets provide a spectrum of complexity, from the relatively simple digit recognition of MNIST \citep{deng2012mnist} to the more challenging fashion object categories of Fashion-MNIST \citep{xiao2017fashion} and the character classes of KMNIST \citep{clanuwat2018deep}. By evaluating across all three, we ensure a comprehensive assessment of model robustness and capacity (detail shown in Table \ref{tab:model_compare}).

\begin{table*}[htb!]
    \centering
    \caption{Comparison of validation accuracy of our model variants on different image classification datasets with existing SOTA ANN and SNN architecture using spike-based backpropagation method (entry marked with N/A indicate data not available). We were able to achieve reasonable performance with significant parameter reduction.}
    \resizebox{\textwidth}{!}{
    \begin{tabular}{|M{6cm}|M{2cm}|*{3}{M{1.8cm}|}} 
    \hline
     \multirow{4}{*}{Model variants} & \multirow{4}{*}{\shortstack[c]{Optimizable\\parameters}} & \multicolumn{3}{c|}{Validation accuracy (\%)} \\
    \cline{3-5}
     & & MNIST \citep{deng2012mnist}& Fashion-MNIST \citep{xiao2017fashion} & KMNIST \citep{clanuwat2018deep}\\
    \hline
    ANN-ResNet18 \citep{he2016deep} & 11.8M & 98.89 & 90.12 & 95.27\\
    \hline
    ANN-ResNet34 \citep{he2016deep} & 21.8M & 98.01 & 91.13 & 90.76\\
    \hline
    \hline
    ST-RSBP (400-R400) \citep{zhang2019spike} & 478,841 &  N/A & 90.13 & N/A\\
    \hline
    ST-RSBP (15C5-P2-40C5-P2-300) \citep{zhang2019spike} & 784,480 & 99.62 & N/A & N/A\\
    \hline
    SEW-ResNet18 (ADD) \citep{fang2021deep} \citep{luu2024improvement} & 11.8M & 98.62 & 88.83 & 94.12\\
    \hline
    LISNN \citep{cheng2020lisnn} & 272,416 & 99.05 & 92.07 & N/A\\
    \hline
    PLIF \citep{fang2021incorporating} & 13.2M &  99.72 & 94.38 & N/A\\
    \hline
    Spiking-ResNet34 \citep{hu2021spiking} & 21.8M &  69.02 & 69.59 & 37.45\\
    \hline
    Spiking-ResNet50 \citep{hu2021spiking} & 25.6M &  13.45 & 12.80 & 10.44\\
    \hline
    \hline
    \textbf{SQDR-CNN[2b-9q] (Ours)} & 803 & 84.03 & 76.18 & 80.04\\
    \hline
    \textbf{SQDR-CNN[4b-9q] (Ours)} & 857 & 88.32 & 79.99 & 84.36\\
    \hline
    \textbf{SQDR-CNN[2b-18q] (Ours)} & 1,190 & 86.55 & 77.80 & 82.01\\
    \hline
    \textbf{SQDR-CNN[4b-18q] (Ours)} & 1,298 & 88.32 & 82.18 & 83.93\\
    \hline
    \end{tabular}}
    \label{tab:model_compare}
\end{table*}

Overall, all SQDR-CNN variants achieve competitive performance while drastically reducing parameter counts. Notably, on MNIST \citep{deng2012mnist} our best model, SQDR-CNN[4b–18q], reaches approximately 86\% of the mean top-performing accuracy of the SOTA baselines (PLIF with 99.72\%), yet uses only 0.5\% of the smallest SOTA model’s parameters (99.5\% size reduction compared to LISNN). It is also noteworthy that all SQDR-CNN variants can outperform the well known architecture Spiking-Resnet variants in term of converged performance with a gap as low as 6.59\% when compare SQDR-CNN[2b-9q] with Spiking-ResNet34 in Fashion-MNIST \citep{xiao2017fashion} and as high as 73.92\% when compare SQDR-CNN[4b-9q] with Spiking-ResNet50 in KMNIST \citep{clanuwat2018deep}. Conversely, the lowest relative performance is observed on Fashion-MNIST \citep{xiao2017fashion}, where SQDR-CNN[2b–9q] attains roughly 78\% of PLIF’s accuracy while employing merely 0.3\% of LISNN’s parameter count (a 99.7\% reduction). Despite this drop, the substantial model size compression achieved across all datasets highlights the potential of SQDR-CNNs for resource‐constrained applications.

\subsection{Benchmarking with noisy quantum gate}

In addition to evaluating performance under ideal (noiseless) conditions, we also conducted benchmarking in the presence of noise, which more accurately reflects the practical operational environment of current quantum hardware. As noted by \citet{chen2023complexity}, we can formally define a NISQ algorithm utilized in quantum computing simulation as follow:
\begin{definition}[NISQ algorithm]
A NISQ$_\lambda$ algorithm with access to $\lambda$-noisy quantum circuits is defined as a probabilistic Turing machine $M$ that can query $\mathrm{NQC}_\lambda$ to obtain an output bitstring $s$ for any number of times, and is denoted as $A_\lambda \triangleq M^{\mathrm{NQC}_\lambda}$. The runtime of $A_\lambda$ is given by the classical runtime of $M$ plus the sum of the times to query $\mathrm{NQC}_\lambda$.
\end{definition}
Aside from being perceivable as probabilistic Turing machine, the NISQ complexity class for decision problems \citep{chen2023complexity} can also be quantified as:
\begin{definition}[NISQ complexity]
A language $L \subseteq \{0,1\}^*$ is in NISQ if there exists a NISQ$_\lambda$ algorithm $A_\lambda$ for some constant $\lambda > 0$ that decides $L$ in polynomial time, that is, such that:
\begin{itemize}[leftmargin =*]
    \item for all $x \in \{0,1\}^*$, $A_\lambda$ produces an output in time $\mathrm{poly}(|x|)$, where $|x|$ is the length of $x$;
    \item for all $x \in L$, $A_\lambda$ outputs $1$ with probability at least $2/3$;
    \item for all $x \notin L$, $A_\lambda$ outputs $0$ with probability at least $2/3$.
\end{itemize}
\end{definition}
Simulated noise models were applied consistently to all applicable gates within the PQC, allowing us to assess the robustness of our architecture under realistic noise conditions.

\subsubsection{Bit flip error channel}
The bit flip error is one of the simplest and most fundamental quantum errors. It models the classical bit-flipping behavior where a qubit in state $\ket{0}$ is flipped to $\ket{1}$, and vice versa. This channel is modeled by the following Kraus matrices:
\begin{equation}
    \begin{aligned}
        K_0 &= \sqrt{1-p} \begin{bmatrix}
                1 & 0 \\
                0 & 1
                \end{bmatrix},\quad K_1 = \sqrt{p}\begin{bmatrix}
                0 & 1  \\
                1 & 0
                \end{bmatrix},\\
    \end{aligned}
\end{equation}
where $p \in [0,1]$ is the probability of a bit flip. The error channel $\varepsilon$ can be represented as $\varepsilon(\rho) = K_0\rho K^{\dag}_0 + K_1\rho K^{\dag}_1$.

\subsubsection{Symmetrically depolarizing error channel}
The depolarizing error is a common quantum noise model that represents the complete loss of information in a qubit, where the state becomes a completely mixed state with some probability. The depolarizing channel can also be written in Kraus operator form using Pauli $X$, $Y$, $Z$ and identity matrices $I$ as following operators:
\begin{equation}
    \begin{aligned}
        K_0 &= \sqrt{1-p} \begin{bmatrix}
                1 & 0 \\
                0 & 1
                \end{bmatrix}, \quad K_1 = \sqrt{\frac{p}{3}}\begin{bmatrix}
                0 & 1  \\
                1 & 0
                \end{bmatrix},\\
        K_2 &= \sqrt{\frac{p}{3}} \begin{bmatrix}
                0 & -i \\
                i & 0
                \end{bmatrix},\quad K_3 = \sqrt{\frac{p}{3}}\begin{bmatrix}
                1 & 0  \\
                0 & -1
                \end{bmatrix},\\
    \end{aligned}
\end{equation}
where $p \in [0,1]$ is the depolarization probability. The error channel $\varepsilon$ can be represented as $\varepsilon(\rho) = \sum_{i=0}^{3} K_i\rho K^{\dag}_i$.
\subsubsection{Amplitude damping channel}
The single-qubit amplitude damping channel models the energy dissipation from a quantum system, commonly used to represent processes like spontaneous emission (e.g., excitepd state $\ket{1}$ decaying to ground state $\ket{0}$). It can be modeled by the amplitude damping channel, with the following Kraus matrices:
\begin{equation}
    \begin{aligned}
        K_0 &= \begin{bmatrix}
                1 & 0 \\
                0 & \sqrt{1-\gamma}
                \end{bmatrix}, \quad K_1 = \begin{bmatrix}
                0 & \sqrt{\gamma}  \\
                0 & 0
                \end{bmatrix},\\
    \end{aligned}
\end{equation}
where $\gamma \in [0,1]$ is the amplitude damping probability. The error channel $\varepsilon$ can be represented as $\varepsilon(\rho) = K_0\rho K^{\dag}_0 + K_1\rho K^{\dag}_1$.

In our experiments, we incorporated gate-level noise into the simulations using a small error probability of \( p \sim 2\%\text{--}3\% \) for all quantum gates. As observed in Table \ref{tab:noise_compare}, the SQDR-CNN[4b–9q] variant suffered the most significant performance degradation under these conditions. This suggests that employing a deep PQC with a limited number of qubits tends to yield suboptimal performance on NISQ devices. This aligns with prior findings in the literature \citep{hanruiwang2022quantumnas}, which highlight how noise and gate errors can introduce large discrepancies between ideal and realistic quantum circuit behavior, particularly in deeper PQCs.

\begin{table}[t]
    \centering
    \caption{Comparison of SQDR-CNN variants validation accuracy with various simulated noises.}
    \resizebox{\textwidth}{!}{
    \begin{tabular}{|M{3.5cm}|M{3cm}|M{3cm}|M{3cm}|} 
    \hline
     \multirow{2}{*}{Variants} & \multicolumn{3}{c|}{Validation accuracy (\%)} \\
    \cline{2-4}
     & MNIST \citep{deng2012mnist} & Fashion-MNIST \citep{xiao2017fashion} & KMNIST \citep{clanuwat2018deep}\\
    \hline
    SQDR-CNN[2b-9q] & 82.67 & 73.86 & 77.54\\
    \hline
    SQDR-CNN[4b-9q] & 81.95 & 72.52 & 77.17\\
    \hline
    SQDR-CNN[2b-18q] & 84.36 & 74.46 & 78.47\\
    \hline
    SQDR-CNN[4b-18q] & \textbf{84.70} &  \textbf{74.97} & \textbf{79.41}\\
    \hline
    \end{tabular}}
    \label{tab:noise_compare}
\end{table}

Interestingly, the accuracy results also indicate that when the model is equipped with a sufficient number of trainable parameters to counteract the effects of gate noise—as exemplified by the SQDR-CNN[4b–18q] variant—the performance impact from simulated noise is significantly mitigated. These results emphasize the importance of balancing PQC depth with qubit count to maintain robustness against hardware-induced noise in NISQ-era quantum machine learning applications.

\subsection{Benchmarking with various optimization algorithms and PQC initialization methods}

In addition to benchmarking our design across a variety of datasets, we conducted a comparative study using several well-known optimization algorithms from the traditional deep learning literature (detail in Table \ref{tab:opt_compare}). Specifically, we evaluated AdamW~ \citep{loshchilov2017decoupled}, stochastic gradient descent (SGD)~ \citep{ruder2016overview}, stochastic gradient descent with warm restarts (SGDR)~ \citep{loshchilov2016sgdr}, and AdaGrad~ \citep{duchi2011adaptive}. For consistency, all optimizers were configured with an initial learning rate of \(\mathrm{lr} = 10^{-3}\) and a decay rate of \(\lambda = 10^{-3}\).

Previous researches \citep{kashif2024alleviating} had empirically shown that normal-like and uniform-like initialization in PQC tend to work better than other distribution due to a higher gradient variance. Accordingly, to assess the impact of parameter initialization on the overall performance of our PQC, we tested several common schemes that had been previously used in deep learning. We sampled from a standard normal distribution \(\mathcal{N}(0,1)\), as well as from uniform distributions \(\mathcal{U}\) over different truncated intervals: 
\begin{itemize}
    \item \(\mathcal{N}(0,1) \in [0, 2\pi]\) and \(\mathcal{U}(0,2\pi)\), corresponding to a full qubit revolution and aligned with common practices in quantum machine learning ~ \citep{bergholm2018pennylane, kashif2024alleviating} and 
    \item \(\mathcal{N}(0,1) \in [0, 1]\) and \(\mathcal{U}(0,1)\), reflecting initialization range previously studied in classical deep learning models \citep{sutskever2013importance}.
\end{itemize}

As shown in Table~\ref{tab:opt_compare}, Adam, the de facto standard optimizer in deep learning, consistently achieved the best performance across all tested initialization strategies. It was followed by AdamW, traditional SGD, SGDR, and AdaGrad. These results highlight the robustness of adaptive moment-based optimizers for training PQCs.

\begin{table*}[t]
    \centering
    \caption{Comparison of SQDR-CNN[2b-9q] validation accuracy on MNIST \citep{deng2012mnist} using various optimization algorithm and PQC truncated distribution initialization. On average, Adam \citep{kingma2014adam} with $\mathcal{U}(0, 2\pi)$ performed the best.}
    \resizebox{\textwidth}{!}{
    \begin{tabular}{|M{5.5cm}*{4}{|M{1.5cm}}|M{2cm}|} 
    \hline
     \multirow{2}{*}{Algorithm} & \multicolumn{5}{c|}{Validation accuracy (\%)} \\
    \cline{2-6}
     & $\mathcal{N}(0, 1) \in [0, 2\pi]$ & $\mathcal{U}(0, 2\pi)$ & $\mathcal{U}[0, 1]$ & $\mathcal{N}(0, 1) \in [0, 1]$ & \textbf{Average (Algorithm wise)}\\
    \hline
    Adam \citep{kingma2014adam} & 81.41 & \textbf{84.03} & 82.57 & 82.07 & \textbf{82.52}\\
    \hline
    AdamW \citep{loshchilov2017decoupled} & 58.40 & \textbf{60.57} & 58.45 & 54.99 & 58.10\\
    \hline
    SGD \citep{ruder2016overview} & 39.93 & \textbf{47.36} & 38.49 & 27.61 & 38.35\\
    \hline
    SGDR \citep{loshchilov2016sgdr} & 25.24 & \textbf{36.96} & 25.25 & 22.41 & 27.47\\
    \hline
    AdaGrad \citep{duchi2011adaptive} & \textbf{31.01} & 28.17 & 25.16 & 20.24 & 26.15\\
    \hline
    \textbf{Average (Initialization wise)} & 47.20 & \textbf{51.42} & 45.98 & 41.46 & \notableentry\\
    \hline
    \end{tabular}}
    \label{tab:opt_compare}
\end{table*}

Among the tested initialization methods, the truncated normal distribution \(\mathcal{U}(0, 2\pi)\) resulted in the highest average accuracy across all optimization algorithms, whereas \(\mathcal{N}(0,1) \in [0, 1]\) yielded the poorest performance. As we investigate further on how the quantum barren plateau affect PQC initialization in Figure \ref{fig:init_compare}, it is empirically clear that PQC gradient while decaying at an almost the same rate for all distribution (except for \(\mathcal{U}(0,1)\) decay slightly faster than other initialization method). \(\mathcal{U}(0, 2\pi)\) is the one with highest mean variance while \(\mathcal{U}(0,1)\) is the one with lowest mean variance upon higher qubit scale, correlated with barren plateau \citep{mcclean2018barren} theorem on PQC convergence. Furthermore, \(\mathcal{N}(0,1) \in [0, 1]\) while having the lowest variance on small qubit count also have a relatively low variance decay rate, indicate that if training with very high qubit count (well over 18 qubit) model might converge better than other distribution.

\begin{figure}[t]
\centerline{\includegraphics[width=\linewidth]{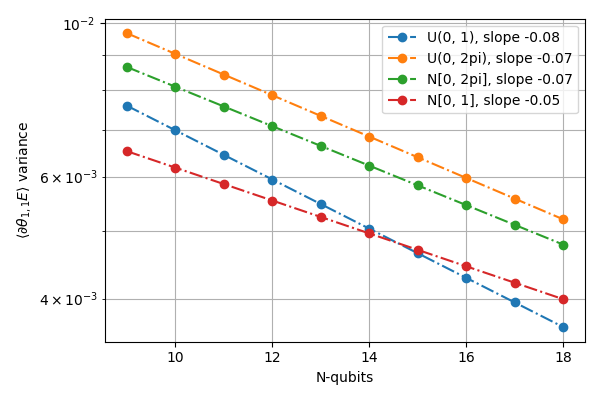}}
\caption{The comparison of mean gradient variance $\langle \partial \theta_{1, 1} E\rangle^2$ of SQDR-CNN[2b-9q] variant's PQC using different initialization methods (higher is better as n-qubit increase).}
\label{fig:init_compare}
\end{figure}
\section{Discussion}
\subsection{Limitations}

While our approach demonstrates promising performance with a minimal number of trainable parameters, several limitations remain evident. First, there is a clear trade-off between PQC size and classification accuracy. Although our model achieves competitive results at smaller scales, scaling to higher qubit counts proves to be nontrivial in both noisy and noiseless simulations. Within noisy simulations, our findings indicate that increasing the depth or width of PQCs does not consistently lead to improved performance. This highlights a key limitation of current hybrid quantum-classical models: their scalability is hindered by noise, barren plateaus, and the lack of efficient circuit design tailored to hybrid learning. Future progress in this area will require the development of more efficient PQC architectures and robust scaling strategies that can mitigate gradient vanishing while maintaining expressivity.

Despite our study provided insights into the potential of hybrid spiking-quantum architectures, our evaluation is limited to simulations and does not include experiments on actual hardware platforms. As a result, the performance of our model has not been validated on real quantum processors or specialized neuromorphic devices such as Intel Loihi or FPGA-based accelerators. Moreover, we did not benchmark against other specialized AI models designed for parameter-efficient learning, which could provide an important point of comparison for assessing scalability and efficiency. Addressing these limitations is an important direction for future work. In particular, more comprehensive benchmarking on practical hardware platforms will be crucial for establishing the applicability of hybrid spiking-quantum systems in real-world scenarios.

\subsection{Future Works}

Future research should focus on the design of hybrid models that scale effectively without disproportionately relying on one component of the architecture, such as the classical SNN encoder or the quantum PQC. Ideally, a balanced model should maintain performance robustness across varying circuit depths and noise levels. Additionally, efforts should be directed toward strategies that inherently reduce or bypass quantum errors and barren plateaus, potentially through the use of noise-resilient circuit templates, adaptive parameter initialization, or error mitigation techniques compatible with current NISQ devices.

\section{Conclusion}

In this work, we proposed SQDR-CNN, an extension of the SQNN framework that leverages spiking CNN encoder as feature extractors and quantum data re-upload as universal classifier. 
Through extensive experiments, our findings can be summarized as:
\begin{itemize}
    \item Proposed a new hybrid spiking-quantum model with quantum data-reupload using recent developments in SNN, allowing a smooth end-to-end optimization process.
    \item Resulted SQDR-CNN variants, while using less parameters than traditional SNN model, are able to deliver competitive performance against SOTA benchmarks without reducing training difficulties.
    \item We also evaluated our proposed model over various condition, including simulated noisy quantum circuit, different initialization settings and optimization algorithm to serve as a reference for future researches.
\end{itemize}  
We hope that our findings contribute to the growing body of research on hybrid quantum models and serve as a stepping stone toward more efficient, scalable, and practical quantum machine learning systems that harness the full potential of emerging quantum technologies.

\section*{Data availability}
Source code for reproducing our experiments is publicly available at \url{https://github.com/luutn2002/SQDR-CNN} (alternatively at Zenodo with DOI: \url{https://doi.org/10.5281/zenodo.17484676}). MNIST, Fashion-MNIST and KMNIST are publicly available at \url{http://yann.lecun.com/exdb/mnist/}, \url{https://github.com/zalandoresearch/fashion-mnist} and \url{https://github.com/rois-codh/kmnist}.

\bibliography{refs}

\end{document}